\documentclass[runningheads]{llncs}
\usepackage{comment}
\usepackage{graphicx}
\usepackage{amsmath}   
\usepackage{amssymb}   
\usepackage{mathtools}   
\usepackage{algorithm}

\usepackage{wrapfig}

\usepackage{cuted}
\usepackage{algorithmic}
\usepackage{graphicx}
\usepackage{textcomp}

\usepackage[hypertexnames=false]{hyperref}
\usepackage{url}

\usepackage{booktabs}
\usepackage{graphicx}

\usepackage{siunitx}
\usepackage{xcolor}
\definecolor{orange}{rgb}{1,0.5,0}

\begin{document}

\title{HD3C: Efficient Medical Data Classification for Edge Devices}

\def\CICAISubNumber{165}  

\titlerunning{HD3C: Efficient Medical Data Classification for Edge Devices}
%

\author{
Jianglan Wei\inst{1,2,}* \and
Zhenyu Zhang\inst{1,}*\textsuperscript{,$\dagger$} \and
Pengcheng Wang\inst{2,}* \and
Mingjie Zeng\inst{1} \and
Zhigang Zeng \inst{1}
}
\authorrunning{J. Wei et al.}
%
\institute{Huazhong University of Science and Technology, Hubei, China \and
University of California, Berkeley, CA, USA}

\maketitle              
\setcounter{footnote}{0}

\vspace{-8pt}

{
\renewcommand{\thefootnote}{*}
\footnotetext{Equal contribution. $\dagger$ Corresponding author: \href{mailto:zzyzhang@hust.edu.cn}{zzyzhang@hust.edu.cn}.}
}

\begin{abstract}
Efficient medical data classification is essential for modern disease screening, particularly in resource-constrained environments where power budgets and computing capabilities are limited. We present HD3C~\footnote{Code: \href{https://github.com/jianglanwei/hd3c}{https://github.com/jianglanwei/hd3c}.}, a lightweight classification framework designed for edge devices. HD3C encodes data into high-dimensional hypervectors, aggregates them into multiple cluster prototypes, and performs classification through similarity search in hyperspace. We evaluate HD3C across three medical classification tasks; on heart sound classification, HD3C is $350\times$ more energy-efficient than Bayesian ResNet with less than 1\% difference in accuracy. Moreover, HD3C demonstrates exceptional robustness to noise, limited training data, and hardware error, supported by both theoretical analysis and empirical results, highlighting its potential for reliable deployment in real-world settings.

\end{abstract}

\vspace{-12pt}

\section{Introduction}

\label{sec:introduction}

Automated medical data classifiers enhance the accessibility of disease screening by reducing its reliance on human doctors~\cite{tan_systematic_2024}. With the growing adoption of portable healthcare devices, disease screening is moving beyond clinical environments into home and field settings, where computational resources and power budgets are limited~\cite{rahman_resource-constrained_2025}. This highlights the need for classification models that can operate effectively on edge platforms~\cite{khan_efficient_2022}.

Deep learning models currently dominate medical classification tasks and achieve state-of-the-art accuracy, but their high energy consumption and reliance on GPUs limit deployment on low-power, near-sensor platforms~\cite{chen_deep_2021}. In contrast, an ideal medical classifier for edge devices should: (1) minimize energy consumption, (2) support GPU-free inference, and (3) process data locally to preserve patient privacy.

To address this, we present Hyperdimensional Computing with Class-Wise Clustering (HD3C), a lightweight framework designed to handle the high variability of real-world clinical data. It extends standard Hyperdimensional Computing (HDC) by introducing a specialized hyperspace clustering layer to better model unstructured data distributions. Specifically, HD3C encodes each sample into a high-dimensional sample hypervector (Sample-HV), aggregates them into a compact set of cluster prototypes (Cluster-HVs), and performs classification by selecting the Cluster-HV with the highest similarity.

Evaluated across three medical classification tasks, HD3C consumes far less energy than deep learning baselines while achieving substantial accuracy gains over standard HDC and other one-shot models. For instance, on the heart sounds classification task, HD3C is 350$\times$ more energy-efficient per inference than the state-of-the-art Bayesian ResNet, while providing $>$10\% accuracy improvement over standard HDC (Figure \ref{fig:teaser}). HD3C also demonstrates exceptional robustness: accuracy drops by only 1.39\% under 15\% input noise, 1.78\% with 40\% training data, and 2.84\% with 20\% parameter corruption (Figure \ref{fig:sensitivity_plot}). We provide theoretical analysis explaining this robustness, supporting HD3C's reliability for real-world deployment. Our main contributions are:

\begin{itemize}
\item[$\bullet$] We propose HD3C, a lightweight classification framework extended from Hyperdimensional Computing (HDC). It provides higher classification accuracy than standard HDC while preserving its energy efficiency. 

\item[$\bullet$] We demonstrate that HD3C consumes far less energy than deep learning baselines across three medical tasks. On heart sound classification, it is $350\times$ more energy-efficient than Bayesian ResNet while maintaining comparable accuracy.

\item[$\bullet$] We provide both theoretical and empirical evidence of HD3C's robustness to input noise and hardware error. To the best of our knowledge, our theorems on binary hypervectors, although some are widely recognized, have not been formally proven in prior HDC research.

\begin{figure*}[t]
\centering
\includegraphics[width=1\linewidth]{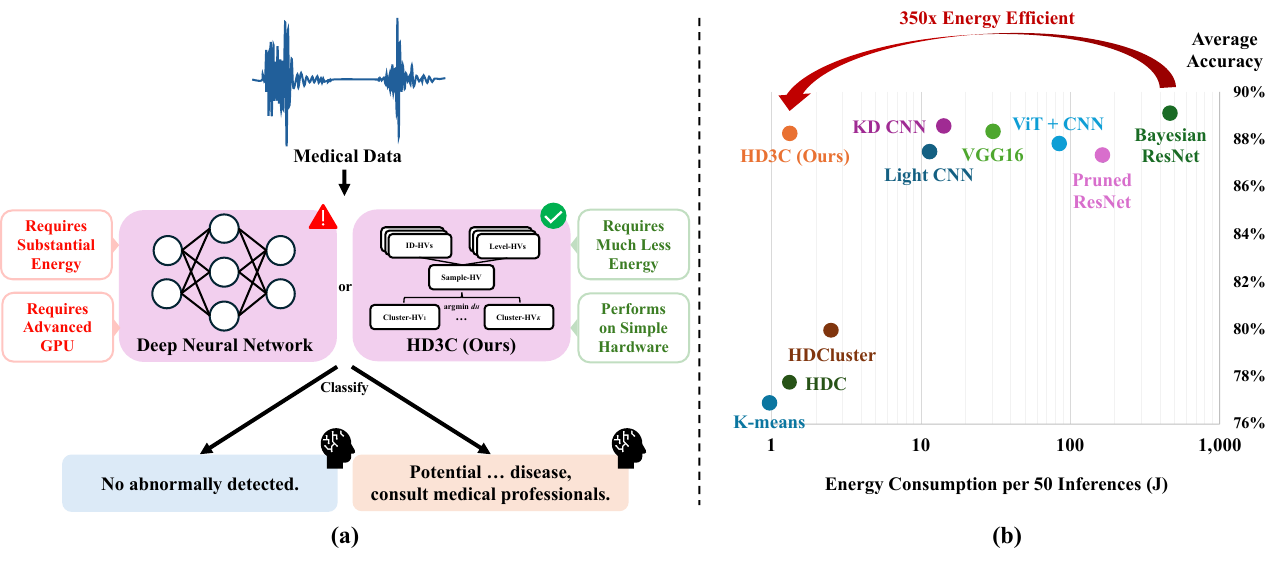}
\caption{\label{fig:teaser} (a) Automated disease screening using different medical data classifiers; (b) HD3C performance on heart sound classification. HD3C is 350$\times$ more energy-efficient than Bayesian ResNet and supports GPU-free inference, highlighting its potential for edge deployment.}
\end{figure*}

\end{itemize}

\section{Hyperdimensional Computing (HDC)}

\label{sec:background}

Hyperdimensional Computing is a computational paradigm inspired by the information processing mechanisms of the brain \cite{pentti_kanerva_hyperdimensional_2009}. Compared with traditional computing that operates on raw numerical data, the human brain processes information via high-dimensional patterns of neural activity~\cite{masse_olfactory_2009}. HDC emulates this approach by projecting input data into high-dimensional representations to efficiently perform cognitive tasks~\cite{thomas_theoretical_2021}.

Specifically, HDC defines a set of elementary operations, like Binding (e.g., point-wise multiplication) and Bundling (e.g., point-wise addition with majority function), on a high-dimensional bipolar vector space $\mathcal{H}^D = \{ -1, +1\}^D$ called hyperspace~\cite{pentti_kanerva_hyperdimensional_2009}. 
The space is equipped with a distance measure called Hamming distance $d_H: \mathcal{H}^D \times \mathcal{H}^D \rightarrow [0, 1]$, which is defined as the ratio of the different bits between two hypervectors (i.e., hyperspace vectors). Please refer to Appendix~\ref{Appendix: Theory} for the detailed notations and definitions.

HDC exhibits several ideal properties when its dimensionality $D$ is very large. For example, the Hamming distance of two hypervectors remains unchanged after bound to the same hypervector (Lemma \ref{lemma: Hamming Distance Preservation under Multiplication}); or when randomly selecting two hypervectors, their Hamming distance is almost always around 0.5 (Lemma \ref{lemma: Hamming distance between two Random Hypervectors}). Based on these properties, HDC can encode a low-dimensional input $s \in \mathbb{R}^d (d \ll D)$ from a continuous feature space into the high-dimensional hyperspace. 

HDC has increasingly demonstrated its efficacy across diverse classification tasks, including image recognition~\cite{billmeyer_biological_2021}, language identification~\cite{alonso_hyperembed_2021}, and acoustic signal processing~\cite{imani_voicehd_2017}. Conventional HDC classification pipelines typically rely on a single hypervector as a centroid prototype to represent the distribution of an entire class, enabling the framework to operate with drastically lower energy and storage than deep neural networks. However, when a single class contains highly heterogeneous samples, such that intra-class variance exceeds inter-class differences, this single-prototype approach becomes brittle. While prior works have explored hyperspace clustering for unsupervised grouping on dispersed data distributions~\cite{imani_hdcluster_2019}, to the best of our knowledge, no existing literature has attempted using hyperspace clustering to enhance a supervised HDC classification framework. 

\section{Medical Data Classification through HD3C}

\label{sec:method}

In this section, we present Hyperdimensional Computing with Class-Wise Clustering (HD3C), a lightweight framework for efficient medical data classification. Algorithm \ref{alg:overall_pipeline} and Figure~\ref{fig:pipeline-overview} illustrate the overall HD3C pipeline. Specifically, HD3C extends standard HDC by introducing a specialized hyperspace clustering layer that can better model the dispersed distributions of real-world clinical datasets (Section~\ref{subsec:method.2}). To maximize performance, an optional retraining stage can be deployed to further enhance accuracy (Section~\ref{subsec:method.3}). Formal definitions, lemmas, and theorems supporting this framework are provided in Appendix~\ref{Appendix: Theory}.

\begin{figure}[t]
\centering
\includegraphics[width=1\linewidth]{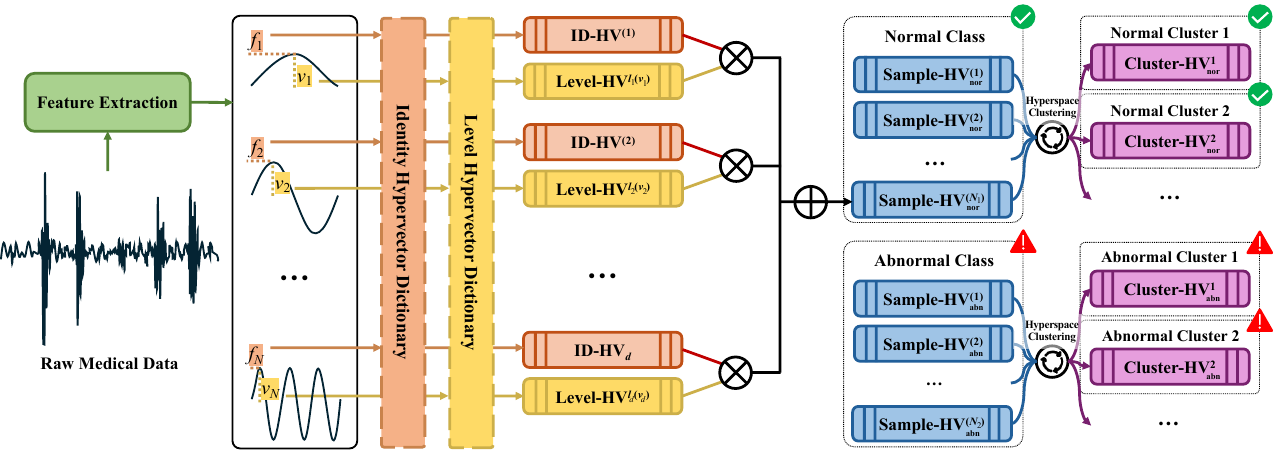}
\caption{\label{fig:pipeline-overview}Medical data classification through HD3C: Training samples are encoded into sample hypervectors (Sample-HVs) and aggregated into a compact set of cluster prototypes (Cluster-HVs); new samples are classified by selecting the Cluster-HV with highest similarity. The figure illustrates a binary classification example, though HD3C is not limited to binary tasks.}
\end{figure}

\subsection{Encode Sample into Hypervector}
\label{subsec:method.1}

HD3C requires numeric features as input. For heart sound classification~\cite{clifford_classification_2016}, we extract $d = 720$ features using Mel-frequency Cepstral Coefficients (MFCC)~\cite{davis_comparison_1980} and Discrete Wavelet Transform (DWT)~\cite{mallat_theory_1989}, two widely used frequency-domain representations in audio analysis. For breast cancer classification~\cite{william_wolberg_breast_1995}, we use $d = 30$ features from fine-needle aspirate (FNA) breast mass images, and for EMG classification, we use $d = 8$ features. Datasets are detailed in Section~\ref{subsec:results.1}.

Each sample's feature vector $s \in \mathbb{R}^d$ is encoded into a high-dimensional binary hypervector called Sample Hypervector (Sample-HV): $S \in \mathcal{H}^D = \{ -1, +1\}^D$, $d \ll D$. The encoder is designed to map similar samples to similar Sample-HVs. The first step of encoding divides each feature's value range into $M$ intervals: the middle $96\%$ of values are split into $M$ equal-width intervals, while the top and bottom $2\%$ are directly mapped to the first and last intervals (Figure \ref{fig:value_mapping}). This is formalized by function $l$ in Definition~\ref{definition: Mapping with Levels}. 

\begin{wrapfigure}{htb}{0.5\linewidth}
    \centering
    \includegraphics[width=\linewidth]{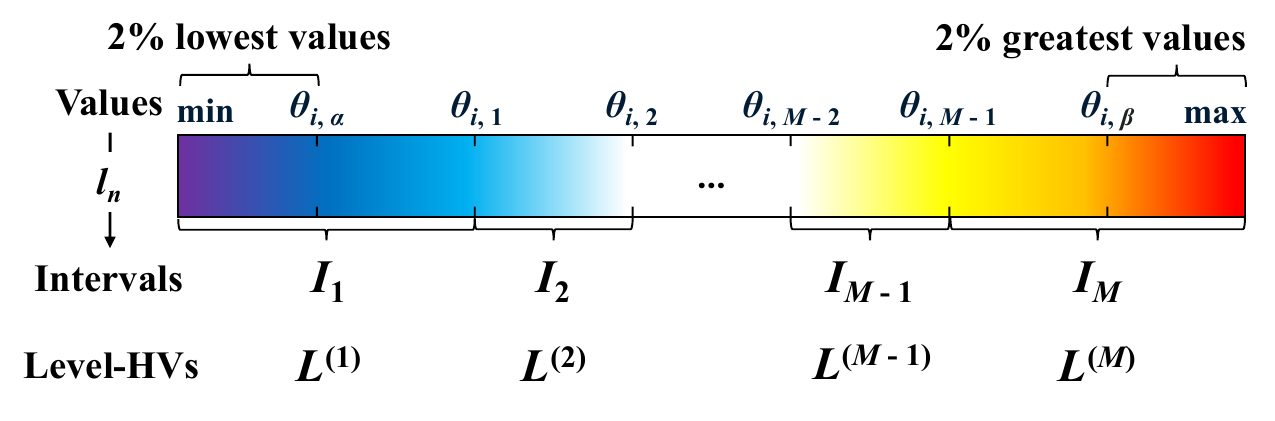}
    \caption{\label{fig:value_mapping}Divide feature $n$'s value range into $M$ intervals.}
\end{wrapfigure}

Each interval in Figure~\ref{fig:value_mapping} is represented by a predefined Level Hypervector (Level-HV), denoted as $L^{(m)} \in \mathcal{H}^D$, $m \in \{1, 2, \ldots, M\}$. The Level-HVs are generated in a way that neighboring Level-HVs have a low Hamming distance (i.e. low percentage of different bits): starting from a randomly generated $L^{(1)}$, each subsequent Level-HV is formed by randomly flipping $D /\ (M - 1)$ bits from the previous vector, with each bit flipped only once across the sequence (Definition~\ref{definition: Level Set}). This ensures that the Hamming distance between any two Level-HVs satisfies:

\begin{equation}
d_H(L^{(i)}, L^{(j)}) = \frac{|i-j|}{M-1}.
\end{equation}

The Level-HVs are shared across all features. To distinguish identical Level-HVs under different features, each feature is assigned a randomly sampled Identity Hypervector (ID-HV), denoted as $ID^{(n)} \in \mathcal{H}^D$, $n \in \{1, 2, \ldots, d\}$. A value in interval $I_m$ under feature $n$ would be represented as $ID^{(n)} \otimes L^{(m)}$, where $\otimes$ denotes the binding operation implemented by point-wise multiplication (Definition~\ref{definition: Elementary Functions}). Since ID-HVs are randomly sampled, their pairwise Hamming distance is approximately 0.5 (Lemma~\ref{lemma: Hamming distance between two Random Hypervectors}), ensuring feature-wise independence: neighboring Level-HVs (representing similar feature values) under the same feature (bound to the same ID-HV) remain similar after binding, whereas Level-HVs bound to different ID-HVs become unrelated (Lemma~\ref{lemma: Hamming Distance Preservation under Multiplication}).

Finally, with the predefined Level-HVs and ID-HVs, the encoder maps each medical sample to a Sample-HV by bundling the representations of each feature:

\begin{equation}
\label{eqn: mapping}
    S^{(i)} = [\sum_{n=1}^d ID^{(n)} \otimes L^{(l_n(s^{(i)}_n))} ],
\end{equation}

\noindent where $s^{(i)}_n$ denotes the value of the $n^{th}$ feature in sample $s^{(i)}$, and $[\cdot]$ is the element-wise majority function: it outputs $+1$ for positive sum, $-1$ for negative sum, and equally random samples from $\{-1, +1\}$ for zero sum (Definition~\ref{definition: Elementary Functions}).

We demonstrate the inherent robustness of this encoding to input noise through Theorem \ref{thm: robust_to_input_noise}. When noise is applied to a feature vector $s^{(1)}$, producing a perturbed version $s^{(2)}$, the theorem shows that as long as the noise is bounded by a relative ratio $\delta$, the Hamming distance between the corresponding Sample-HVs $S^{(1)}$ and $S^{(2)}$ has an upper bound. This upper bound, expressed as $g(\delta)$, remains quantitatively small even under moderate noise. The formal proof and definition of $g(\delta)$ are provided in Proof~\ref{prf:robust_to_input_noise}.

\begin{theorem}[Robustness to Input Noise]
\label{thm: robust_to_input_noise}
Let $s^{(1)}$ be a feature vector and $s^{(2)}$ its noisy variant, with their Sample-HVs denoted as $S^{(1)}, S^{(2)} \in \{-1, 1\}^D$ according to Equation~\ref{eqn: mapping}.
Suppose that for feature dimensions $n \in \{1, \dots, d\}$,
\begin{equation}
\frac{\left| s_n^{(1)} - s_n^{(2)} \right|}{\Delta_n} \leq \delta,
\end{equation}
where $\Delta_n$ is the range of the $n^{\mathrm{th}}$ feature value, and $\delta \in [0, 1]$ denotes the maximum normalized perturbation.
Then, with a sufficiently large $D$, the expected upper-bound of the Hamming distance between $S^{(1)}$ and $S^{(2)}$ converges to a monotonically increasing function $g(\delta)$:
\begin{equation}
    \lim_{D \to \infty} \mathbb{E}\left[\sup d_H\left( S^{(1)}, S^{(2)} \right) \right] = g(\delta).
\end{equation}
\end{theorem}

\subsection{Class-Wise Hyperspace Clustering}
\label{subsec:method.2}

Medical data often exhibit substantial intra-class variability beyond simple class labels. For example, heart sounds sharing the same ‘abnormal’ label can differ due to the type and stage of cardiac disease, the stethoscope used, and the recording site on the body~\cite{zipes_braunwalds_2019}. Such heterogeneity challenges the standard HDC pipeline, which rely on a single prototype to represent all samples within one class.

To address this, HD3C introduces class-wise hyperspace clustering. The clustering process is inspired by K-means, but performed in hyperspace on Sample-HVs and applied independently within each class. Unlike K-means which computes arithmetic means, HD3C uses the bundling operation (Appendix A.1 Definition 2) to form cluster prototypes (Cluster-HVs). Leveraging hypervectors rather than raw numeric features provides the robustness formalized in Theorem \ref{thm: robust_to_input_noise}. 

Specifically, for each class $j$, clustering begins by randomly assigning its Sample-HVs $\{S_j^{(i)}\}$ to $K$ clusters $\{\mathfrak{C}_j^k\}^K_{k=1}$. Each Cluster-HV $C_j^k \in \mathcal{H}^D$ is computed by bundling the Sample-HVs in that cluster:

\begin{equation}
\label{eqn: clustering}
C^k_j = [\sum_{S_j^{(i)} \in \mathfrak{C}^k_j} S_j^{(i)}].
\end{equation}

\noindent Next, the Sample-HVs are reassigned to the cluster whose prototype has the lowest Hamming distance:

\begin{equation}
S_j^{(i)} \rightarrow \operatorname*{argmin}_{k \in [K]}d_H(S_j^{(i)}, C_j^k), \qquad [K] \coloneqq \{1, 2, \dots, K\}.
\end{equation}

This bundling and reassigning process is repeated for $T$ iterations until convergence, yielding the final set of Cluster-HVs $\{{C}_j^k\}^K_{k=1}$ for class $j$. After clustering all $J$ classes, HD3C may classify an unseen new sample by encoding it as Sample-HV $S^\mathrm{new}$ and assigning it to the Cluster-HV with the highest similarity (i.e., lowest Hamming distance):

\begin{equation}
S^\mathrm{new} \rightarrow \operatorname*{argmin}_{(j, k) \in [J] \times [K]} d_H(S^\mathrm{new}, C_j^k).
\end{equation}

\noindent The sample is then classified according to the class label of the selected cluster.

Theorem~\ref{thm: proximity_guarantee} provides theoretical insight into our clustering-by-bundling method in Equation \ref{eqn: clustering} (Proof~\ref{prf:proximity_guarantee}). It shows that the Hamming distance between the Cluster-HV $C$ and any Sample-HV that constitutes the cluster is almost always less than the typical distance of $d_H=0.5$ between two random hypervectors. This supports that the Cluster-HV can sufficiently represent all samples in the cluster. Moreover, since Sample-HVs within a cluster are typically more similar to one another than two randomly drawn hypervectors, the Cluster-HV is expected to preserve even stronger relationships in practice than those guaranteed under the random hypervector assumption in the theorem.

\begin{theorem}[Distance Between Cluster Prototype and Constituents]
\label{thm: proximity_guarantee}
Let $S^{(1)}, S^{(2)}, \dots, S^{(N)} \in \mathcal{H}^D$ be independently sampled random hypervectors. Define their bundling sum as $C = [S^{(1)} + S^{(2)} + \dots + S^{(N)}]$. 
As $D \to \infty$, for any random hypervector $S^* \in \mathcal{H}^D$, index $n \in \{1, \dots, N\}$, the Hamming distance between $C$ and any component $S^{(n)}$ satisfies 

\begin{equation}
P\left(d_H(C, S^{(n)}) < d_H(C, S^*)\right) \rightarrow 1.
\end{equation}

\end{theorem}

\begin{algorithm*}[t]
    \caption{The HD3C Framework (Without Retraining)}
    \label{alg:overall_pipeline}
    \begin{algorithmic}[1]
        \STATE $\{ID^{(n)}\}_{n=1}^{d}\leftarrow \mathrm{identity\_hypervectors}(d)$ \hfill $\triangleright$ Generate ID-HV Dictionary.
        \STATE $\{L^{(m)}\}_{m=1}^{M} \leftarrow \mathrm{level\_hypervectors}(M)$ \hfill $\triangleright$ Generate Level-HV Dictionary.
        \FOR{$s^{(i)} \in \mathcal{S}$ \textbf{in} training set}
            \STATE $S^{(i)} \leftarrow [\sum_{n=1}^d ID^{(n)} \otimes L^{(l_n(s_n^{(i)}))}) ]$ \hfill $\triangleright$ Encode sample $s^{(i)}$ into Sample-HV $S^{(i)}$.
        \ENDFOR
        \FOR{$j = 1$ \textbf{to} number of classes $J$}
            \FOR{$S_j^{(i)} \in$ class $j$}
                \STATE $\mathrm{init\_cluster\_idx} \leftarrow \mathrm{randint}(K)$ \hfill $\triangleright$ Random assign $S_j^{(i)}$ to one of $K$ clusters.
                \STATE $\mathfrak{C}_j^{\mathrm{init\_cluster\_idx}}\leftarrow \mathfrak{C}_j^{\mathrm{init\_cluster\_idx}} \cup \{S_j^{(i)}\}$
            \ENDFOR
            \FOR{$t = 1$ \textbf{to} number of clustering iterations $T$}
                \FOR{$k = 1$ \textbf{to} number of clusters $K$}
                    \STATE $C_j^{k} \leftarrow  [\sum_{S_j^{(i)} \in \mathfrak{C}_j^k} S_j^{(i)}]$ \hfill $\triangleright$ Generate Cluster-HV $C^k_j$ to represent cluster $\mathfrak{C}_j^k$.
                \ENDFOR
                \FOR{$S_j^{(i)} \in$ class $j$}
                    \STATE $\mathrm{cluster\_idx} \leftarrow \operatorname*{argmin}_{k=1}^{K}d_H(S_j^{(i)}, C_j^k$) \hfill $\triangleright$ Reassign $S_j^{(i)}$ to closest cluster.
                    \STATE $\mathfrak{C}_j^{\mathrm{cluster\_idx}} \leftarrow \mathfrak{C}_j^{\mathrm{cluster\_idx}}  \cup \{S_j^{(i)}\}$
                \ENDFOR
            \ENDFOR
        \ENDFOR
        \STATE $(\mathrm{class\_idx}, \mathrm{cluster\_idx}) \leftarrow \operatorname*{argmin}_{(j, k) \in [J] \times [K]} d_H(S_j^{\mathrm{new}}, C_j^k)$ \hfill $\triangleright$ Classify new sample.
        \RETURN $\mathrm{class\_idx}$
    \end{algorithmic}
\end{algorithm*}

\subsection{Retrain Cluster Prototypes}
\label{subsec:method.3}

To further improve accuracy, HD3C optionally applies a retraining procedure that adjusts Cluster-HVs based on misclassified training samples. The retraining stage operates exclusively on the training set with no reference to test data.

Recall that each Cluster-HV $C_j^k$ is generated by bundling the Sample-HVs assigned to cluster $\mathfrak{C}_j^k$, i.e., $C_j^k = [\sum_{S_j^{(i)} \in \mathfrak{C}^k_j} S_j^{(i)}]$. Let $S^\mathrm{err}$ be a misclassified Sample-HV from the training set. Suppose $S^\mathrm{err} \in \mathfrak{C}_{j_1}^{k_1}$, but its nearest Cluster-HV in Hamming distance is a different Cluster-HV $C_{j_2}^{k_2}$. This implies that the Cluster-HV $C_{j_1}^{k_1}$, constructed from cluster $\mathfrak{C}_{j_1}^{k_1}$, does not adequately represent $S^\mathrm{err}$.

To correct this, HD3C performs two adjustments: (1) Subtracts $S^\mathrm{err}$ from the incorrect Cluster-HV $C_{j_2}^{k_2}$ and (2) Re-bundles $S^\mathrm{err}$ onto the correct Cluster-HV $C_{j_1}^{k_1}$. 
We generalize this operation across all misclassified training samples. Let $\mathcal{E}_{j}^\mathrm{k, out}$ denote Sample-HVs outside cluster $\mathfrak{C}_j^k$ that incorrectly match closest to $C_j^k$; while $\mathcal{E}_{j}^\mathrm{k, in}$ denote the Sample-HVs within $\mathfrak{C}^j_k$ but are closer to a Cluster-HV of a different cluster. The retrained Cluster-HV representing cluster $\mathfrak{C}^j_k$ is computed as

\begin{equation}
C_j^{k^\prime} = [\sum_{S_j^{(i)} \in \mathfrak{C}_j^k} S_j^{(i)} - \sum_{S_j^{(i)} \in \mathcal{E}_{j}^\mathrm{k, out}} S_j^{(i)} + \sum_{S_j^{(i)} \in \mathcal{E}_{j}^\mathrm{k, in}} S_j^{(i)}].
\end{equation}

\section{Results and Discussions}

\label{sec:results}

\subsection{Datasets}

\label{subsec:results.1}

\noindent \textbf{PhysioNet/CinC Challenge 2016}~\cite{clifford_classification_2016} provides a collection of 3,153 heart sound recordings sourced from six distinct databases. These recordings were collected by multiple research teams across various countries using different equipment and methodologies, often under noisy conditions. Each recording is labeled as either ‘normal’ or ‘abnormal’.

\vspace{1em}

\noindent \textbf{Wisconsin Breast Cancer}~\cite{william_wolberg_breast_1995} is a widely used benchmark for breast cancer diagnosis. It contains real-valued features extracted from digitized images of fine-needle aspirates (FNA) of breast masses, with each sample labeled as either ‘benign’ or ‘malignant’.

\vspace{1em}

\noindent \textbf{sEMG Muscle Fatigue} provides surface electromyography (sEMG) recordings at a 2000Hz sampling rate. Each sample is annotated into three fatigue levels: relax, mild fatigue, and fatigue.

\subsection{Comparison with Baseline Models}

Table~\ref{tab:performance_compare} summarizes the performance of HD3C against baseline models across the three medical datasets. The energy listed in the table was computed by integrating measured power over time. HD3C consumes far less energy than deep neural networks while outperforming one-shot models in terms of accuracy. For instance, on the PhysioNet heart sounds classification task, HD3C is $580\times$ more energy-efficient in training and $350\times$ in inference compared to the best-performing Bayesian ResNet~\cite{h_krones_dual_2022}, while also delivering a 10\% accuracy improvement over standard HDC.

\begin{table*}[t]
\centering
\caption{Performance and Energy Comparison of HD3C and Baseline Models}
\label{tab:performance_compare}
\resizebox{\textwidth}{!}{%
\begin{tabular}{lrrrr}
\toprule
\textbf{Models} & \textbf{10-Folds Accuracy} (\%) & \textbf{Energy (Train, J)} & \textbf{Energy (1k Inferences, J)}\\
\midrule
\textbf{PhysioNet Challenge 2016:} \\
Bayesian ResNet \cite{h_krones_dual_2022} & 89.105 ± 1.543 & 142997 ± \: 5465 & 9455 ± 899\\
Knowledge Distillation CNN \cite{song_cutting_2023} & 88.580 ± 2.186 & 32808 ± \: 2582 & 289 ± 116\\
VGG16 \cite{shuvo_nrc-net_2023} & 88.271 ± 1.718 & 165840 ± \: 5443 & 605 ± 131\\
ViT + CNN \cite{han_enact-heart_2025} & 87.808 ± 1.996 & 889920 ± 36748 & 1661 ± 209\\
Light CNN \cite{li_lightweight_2021} & 87.408 ± 1.497 & 10164 ± \: 1204 & 227 ± \: 88\\
Pruned Bayesian ResNet \cite{h_krones_dual_2022} & 87.190 ± 2.508 & 142997 ± \: 5465 & 3307 ± 214\\
Class-Wise K-Means & 76.759 ± 4.937 & 9 ± \qquad 2 & \textit{very low}\\
HDC & 77.840 ± 2.419  & 135 ± \:\;\quad 10 & 26 ± \:\:\: 4\\
\textcolor{orange}{\textbf{HD3C (Ours)}} & 88.180 ± 1.746 & 246 ± \qquad 6 & 27 ± \:\:\: 3 \\
\midrule
\textbf{Wisconsin Breast Cancer:} \\
DNN \cite{zheng_training_2024} & 96.842 ± 2.579 & 698.626 ± \: 82.403 & 11.842 ± 1.470 \\
GRU \cite{jony_deep_2024} & 96.668 ± 1.994 & 2710.461 ± 109.069 & 97.283 ± 2.408 \\
CNN \cite{jony_deep_2024} & 95.439 ± 2.956 & 312.416 ± \:\:\: 7.874 & 12.458 ± 0.368 \\
Class-Wise K-Means & 93.595 ± 3.106 & 3.045 ± \:\:\: 0.379 & 0.070 ± 0.007 \\
HDC & 94.382 ± 0.039  & 2.489 ± \:\:\:  0.174 & 0.717 ± 0.021\\
\textcolor{orange}{\textbf{HD3C (Ours)}} & 96.314 ± 0.027 & 3.801 ± \:\:\: 0.340 & 0.823 ± 0.015 \\
\midrule
\textbf{sEMG Muscle Fatigue: }\\
GRU \cite{aviles_optimizing_2024} & 91.955 ± 3.084 & 874.511 ± 42.378 & 34.018 ± 1.013 \\
LSTM \cite{aviles_optimizing_2024} & 91.710 ± 2.705 & 135.828 ± \: 7.453 & 32.416 ± 1.483 \\
CNN \cite{moniri_real-time_2021} & 91.367 ± 3.626 & 250.685 ± 21.217 & 10.426 ± 0.884 \\
Class-Wise K-Means & 88.596 ± 5.135 & 3.367 ± \: 0.247 & 0.076 ± 0.010 \\
HDC \cite{moin_emg_2018} & 84.984 ± 0.043 & 2.634 ± \: 0.169 & 0.445 ± 0.025 \\
\textcolor{orange}{\textbf{HD3C (Ours)}} & 91.592 ± 2.927 & 2.813 ± \: 0.205 & 0.422 ± 0.029 \\
\bottomrule

\end{tabular}%
}
\end{table*}

\subsection{Impact of HD3C Hyperparameters}

As discussed in Section~\ref{sec:method}, a sufficiently large dimensionality $D$ is critical for HD3C to maintain robustness against input noise (Theorem~\ref{thm: robust_to_input_noise}) and to ensure that Cluster-HVs accurately capture the aggregate features of their clusters (Theorem~\ref{thm: proximity_guarantee}). Likewise, an adequate number of clusters $K$ are essential for forming stable hyperspace clusters. A few retrain epochs can further fine-tune Cluster-HVs to better align with the Sample-HVs assigned to each cluster.

However, excessively large values of these hyperparameters introduce unnecessary computational overhead, reducing the energy efficiency HD3C is designed to achieve. Our experiments also indicate that allocating too many clusters or performing excessive retraining can lead to overfitting. Figure~\ref{fig:sensitivity_plot} (a-c) illustrates HD3C's performance on heart sound classification across different hyperparameter settings.

\subsection{Robustness for Real-World Deployment}

\begin{figure*}[t]
\centering
\includegraphics[width=\linewidth]{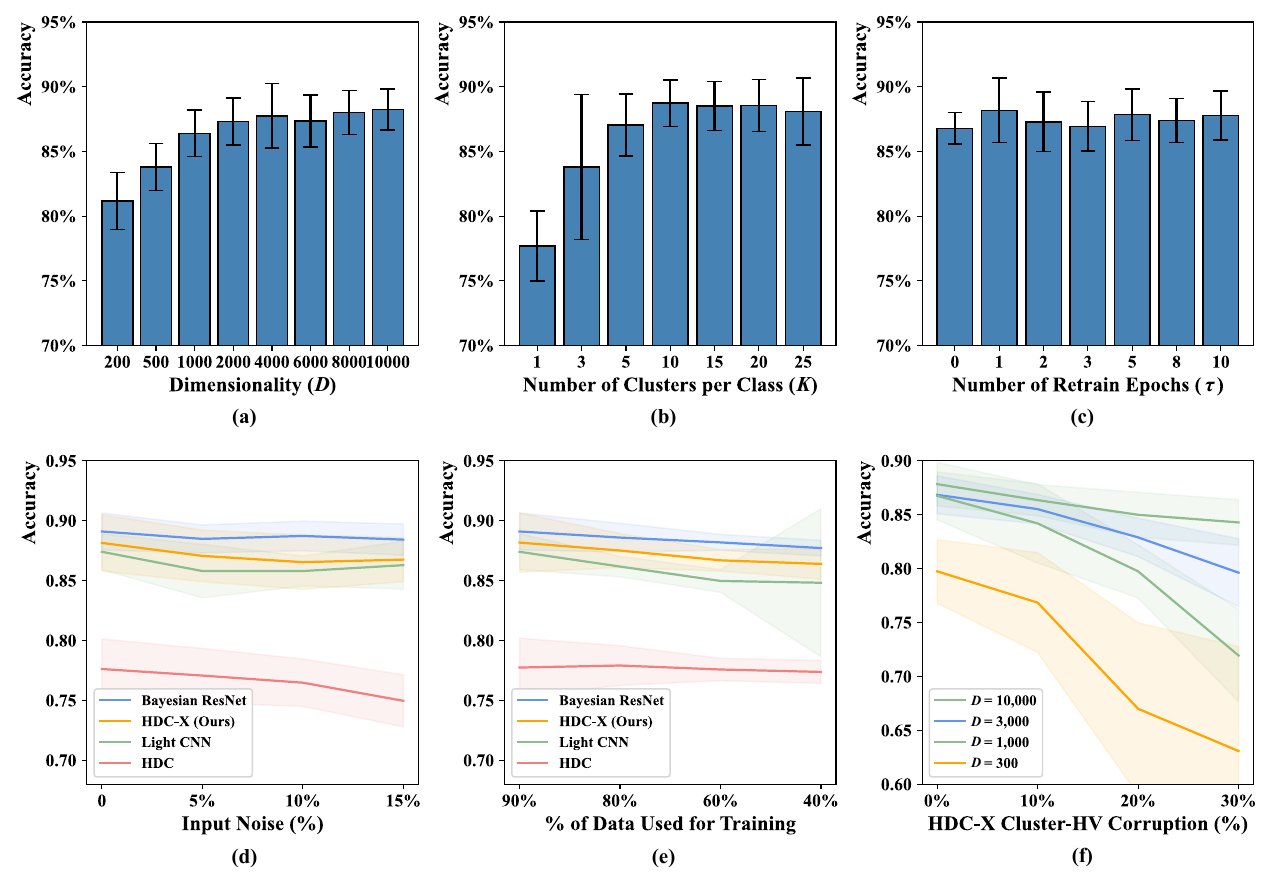}
\caption{\label{fig:sensitivity_plot} HD3C sensitivity to hyperparameters, input noise, limited training data, and hardware errors on PhysioNet 2016.}
\end{figure*}

\noindent \textbf{Resilience to Input Noise.} Real-world medical signals often contain persistent environmental noise, which can impair classification performance~\cite{clifford_classification_2016}. The inherent robustness of HD3C's encoding mechanism, as supported by Theorem~\ref{thm: robust_to_input_noise}, allows it to maintain high accuracy even with noisy inputs. As shown in Figure~\ref{fig:sensitivity_plot}(d), HD3C experiences only 1.39\% drop in accuracy under 15\% input noise, demonstrating its reliability in noisy settings.

\vspace{1em}

\noindent \textbf{Resilience to Limited Training Data.} Many medical datasets are limited in size due to privacy constraints. As shown in Figure~\ref{fig:sensitivity_plot}(e), HD3C remains robust under reduced training data, with only a 1.78\% drop in accuracy when trained on 40\% of the PhysioNet 2016 dataset.

\vspace{1em}

\noindent \textbf{Resilience to Hardware Errors.} Neural architectures are known for their fault tolerance through redundant representations, unlike traditional binary systems where single-bit failures are critical~\cite{pentti_kanerva_hyperdimensional_2009}. Similarly, brain-inspired hyperdimensional encodings offer inherent robustness to hardware malfunctions by distributing information across high-dimensional vectors~\cite{kanerva_random_2000}. 

To evaluate this robustness in HD3C, we conducted a perturbation analysis by randomly flipping elements (+1 to -1 and vice versa) in all stored Cluster-HVs to simulate hardware instability. Theorem~\ref{thm: robust_to_hardware_error} demonstrates that, with sufficient dimensionality, flipping up to 50\% of the elements has minimal impact on classification accuracy (Proof~\ref{prf:robust_to_hardware_error}). The theorem is supported empirically: as shown in Figure \ref{fig:sensitivity_plot}(f), with $D = 10{,}000$, flipping 20\% of elements results in only a 2.84\% drop in accuracy.

\begin{theorem}[Robustness to Hardware Error]
\label{thm: robust_to_hardware_error}
Assume we have a sample hypervector $S$ and two cluster hypervectors $C_1$ and $C_2$, whose initial Hamming distances satisfy:
\[
d_H(S, C_2) - d_H(S, C_1) = \epsilon > 0.
\]

We randomly flip a proportion $p$ ($p < 0.5$) of the bits in both $C_1$ and $C_2$, yielding two new hypervectors $C'_1$ and $C'_2$. As $D \rightarrow \infty$, we have

\[
P\big( d_H(S, C'_1) < d_H(S, C'_2) \big) \to 1.
\]
\end{theorem}

\subsection{Conceptual Hardware Framework}

\begin{wrapfigure}{h}{0.56\linewidth}
    \centering
    \includegraphics[width=\linewidth]{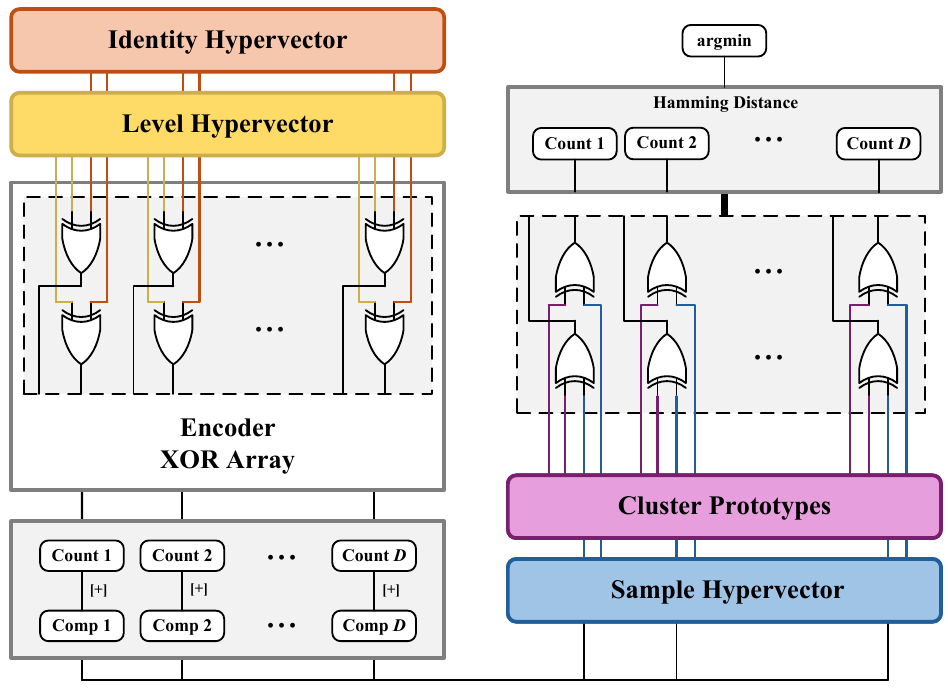}
    \caption{\label{fig:hardware} Conceptual hardware framework for HD3C.}
\end{wrapfigure}

The current implementation of HD3C is in Python, which does not fully exploit the hardware-level efficiency of binary hypervectors enabled by single-bit operations and parallel computing. Prior work on HDC hardware has shown that fundamental operations such as binding and bundling can be implemented with extremely lightweight digital logic: binding reduces to bitwise~\texttt{XOR}, bundling to majority voting circuits, and similarity search to parallelized Hamming distance computation~\cite{imani_voicehd_2017}. These properties allow HDC systems to operate with low energy consumption and high throughput compared to conventional floating-point ML models. 

Building on these insights, we outline a conceptual hardware framework specifically tailored for HD3C (Figure~\ref{fig:hardware}). To validate its feasibility, we prototyped the framework on a Xilinx FPGA platform and are able to achieve the accuracy reported in Table~\ref{tab:performance_compare}. We plan to further refine and evaluate this hardware-oriented design in future work.

\section{Conclusion}

\label{sec:conclusion}

This paper introduces HD3C, an energy-efficient classification framework extended from Hyperdimensional Computing (HDC). HD3C demonstrates significant advantage in medical data classification: on heart sound classification, it is $350\times$ more efficient than Bayesian ResNet, and provide $>10\%$ accuracy improvement over standard HDC.

We evaluated HD3C's robustness both theoretically and empirically. The model maintains high accuracy under challenging conditions: showing only a 1.39\% drop under 15\% input noise, 1.78\% when trained on just 40\% of the dataset, and 2.84\% when 20\% of its learned parameters are corrupted. These results highlight HD3C's reliability on real-world medical applications.

From a societal perspective, HD3C can help expand access to medical screening in underserved settings by enabling assessments on low-cost, GPU-free devices. However, overreliance on the screening tool without clinical oversight could lead to misdiagnoses; therefore, the system should be designed to recommend follow-up with a medical professional when abnormalities are detected. 

%
%
%

\section*{Acknowledgments}
This work was supported by the National Key R\&D Program of China (Grant No. 2021ZD0201300), the National Natural Science Foundation of China (Grant Nos. 62406118), the Postdoctoral Project of Hubei Province (Grant No. 2024HBBHCXA017), the Key Research and Development Program of Wuhan City (Grant No. 2025051202030404), and the Opening Project of Hubei Key Laboratory of Cognitive and Affective Disorder (Grant No. HBCAD2025-12).

\bibliography{references}

@inproceedings{imani_voicehd_2017,
	address = {Washington, DC},
	title = {{VoiceHD}: {Hyperdimensional} {Computing} for {Efficient} {Speech} {Recognition}},
	isbn = {978-1-5386-1553-9},
	shorttitle = {{VoiceHD}},
	language = {en},
	urldate = {2024-07-10},
	booktitle = {2017 {IEEE} {International} {Conference} on {Rebooting} {Computing} ({ICRC})},
	publisher = {IEEE},
	author = {Imani, Mohsen and Kong, Deqian and Rahimi, Abbas and Rosing, Tajana},
	month = nov,
	year = {2017},
	pages = {1--8},
}

@inproceedings{alonso_hyperembed_2021,
	address = {Shenzhen, China},
	title = {{HyperEmbed}: {Tradeoffs} {Between} {Resources} and {Performance} in {NLP} {Tasks} with {Hyperdimensional} {Computing} {Enabled} {Embedding} of n-gram {Statistics}},
	isbn = {978-1-6654-3900-8},
	shorttitle = {{HyperEmbed}},
	language = {en},
	urldate = {2024-07-15},
	booktitle = {2021 {International} {Joint} {Conference} on {Neural} {Networks} ({IJCNN})},
	publisher = {IEEE},
	author = {Alonso, Pedro and Shridhar, Kumar and Kleyko, Denis and Osipov, Evgeny and Liwicki, Marcus},
	month = jul,
	year = {2021},
	pages = {1--9},
}

@inproceedings{billmeyer_biological_2021,
	address = {Pacific Grove, CA, USA},
	title = {Biological {Gender} {Classification} from {fMRI} via {Hyperdimensional} {Computing}},
	copyright = {https://doi.org/10.15223/policy-029},
	language = {en},
	urldate = {2024-07-15},
	booktitle = {2021 55th {Asilomar} {Conference} on {Signals}, {Systems}, and {Computers}},
	publisher = {IEEE},
	author = {Billmeyer, Ryan and Parhi, Keshab K.},
	month = oct,
	year = {2021},
	pages = {578--582},
}

@inproceedings{moin_emg_2018,
	address = {Florence},
	title = {An {EMG} {Gesture} {Recognition} {System} with {Flexible} {High}-{Density} {Sensors} and {Brain}-{Inspired} {High}-{Dimensional} {Classifier}},
	isbn = {978-1-5386-4881-0},
	language = {en},
	urldate = {2024-07-15},
	booktitle = {2018 {IEEE} {International} {Symposium} on {Circuits} and {Systems} ({ISCAS})},
	publisher = {IEEE},
	author = {Moin, Ali and Zhou, Andy and Rahimi, Abbas and Benatti, Simone and Menon, Alisha and Tamakloe, Senam and Ting, Jonathan and Yamamoto, Natasha and Khan, Yasser and Burghardt, Fred and Benini, Luca and Arias, Ana C. and Rabaey, Jan M.},
	year = {2018},
	pages = {1--5},
}

@article{khan_efficient_2022,
	title = {An {Efficient}, {Ensemble}-{Based} {Classification} {Framework} for {Big} {Medical} {Data}},
	volume = {10},
	issn = {2167-6461, 2167-647X},
	language = {en},
	number = {2},
	urldate = {2025-08-01},
	journal = {Big Data},
	author = {Khan, Firoz and Siva Prasad, Balusupati Veera Venkata and Syed, Salman Ali and Ashraf, Imran and Ramasamy, Lakshmana Kumar},
	month = apr,
	year = {2022},
	pages = {151--160},
}

@article{rahman_resource-constrained_2025,
	title = {Resource-{Constrained} {On}-{Chip} {AI} {Classifier} for {Beat}-by-{Beat} {Real}-{Time} {Arrhythmia} {Detection} with an {ECG} {Wearable} {System}},
	volume = {14},
	issn = {2079-9292},
	language = {en},
	number = {13},
	urldate = {2025-08-01},
	journal = {Electronics},
	author = {Rahman, Mahfuzur and Morshed, Bashir I.},
	month = jun,
	year = {2025},
	pages = {2654},
}

@article{chen_deep_2021,
	title = {Deep {Learning} on {Mobile} and {Embedded} {Devices}: {State}-of-the-art, {Challenges}, and {Future} {Directions}},
	volume = {53},
	issn = {0360-0300, 1557-7341},
	shorttitle = {Deep {Learning} on {Mobile} and {Embedded} {Devices}},
	language = {en},
	number = {4},
	urldate = {2025-08-01},
	journal = {ACM Computing Surveys},
	author = {Chen, Yanjiao and Zheng, Baolin and Zhang, Zihan and Wang, Qian and Shen, Chao and Zhang, Qian},
	month = jul,
	year = {2021},
	pages = {1--37},
}

@article{tan_systematic_2024,
	title = {A systematic review of the impacts of remote patient monitoring ({RPM}) interventions on safety, adherence, quality-of-life and cost-related outcomes},
	volume = {7},
	issn = {2398-6352},
	language = {en},
	number = {1},
	urldate = {2025-08-01},
	journal = {npj Digital Medicine},
	author = {Tan, Si Ying and Sumner, Jennifer and Wang, Yuchen and Wenjun Yip, Alexander},
	month = jul,
	year = {2024},
	pages = {192},
}

@article{zheng_training_2024,
	title = {Training {Artificial} {Neural} {Network} for {Breast} {Cancer} {Detection}: a {High} {Accuracy} {Model} to {Compare} {Different} {Features} of {Breast} {Cell} {Nuclei} in {Breast} {Cancer} {Diagnosis}},
	language = {en},
	author = {Zheng, Lily},
	month = oct,
	year = {2024},
}

@article{aviles_optimizing_2024,
	title = {Optimizing {RNNs} for {EMG} {Signal} {Classification}: {A} {Novel} {Strategy} {Using} {Grey} {Wolf} {Optimization}},
	volume = {11},
	issn = {2306-5354},
	shorttitle = {Optimizing {RNNs} for {EMG} {Signal} {Classification}},
	language = {en},
	number = {1},
	urldate = {2025-07-30},
	journal = {Bioengineering},
	author = {Aviles, Marcos and Alvarez-Alvarado, José Manuel and Robles-Ocampo, Jose-Billerman and Sevilla-Camacho, Perla Yazmín and Rodríguez-Reséndiz, Juvenal},
	month = jan,
	year = {2024},
	pages = {77},
}

@article{moniri_real-time_2021,
	title = {Real-{Time} {Forecasting} of {sEMG} {Features} for {Trunk} {Muscle} {Fatigue} {Using} {Machine} {Learning}},
	volume = {68},
	copyright = {https://ieeexplore.ieee.org/Xplorehelp/downloads/license-information/IEEE.html},
	issn = {0018-9294, 1558-2531},
	language = {en},
	number = {2},
	urldate = {2025-07-30},
	journal = {IEEE Transactions on Biomedical Engineering},
	author = {Moniri, Ahmad and Terracina, Dan and Rodriguez-Manzano, Jesus and Strutton, Paul H. and Georgiou, Pantelis},
	month = feb,
	year = {2021},
	pages = {718--727},
}

@article{jony_deep_2024,
	title = {Deep {Learning} {Paradigms} for {Breast} {Cancer} {Diagnosis}: {A} {Comparative} {Study} on {Wisconsin} {Diagnostic} {Dataset}},
	issn = {2785-8901},
	shorttitle = {Deep {Learning} {Paradigms} for {Breast} {Cancer} {Diagnosis}},
	language = {en},
	urldate = {2025-07-30},
	journal = {Malaysian Journal of Science and Advanced Technology},
	author = {Jony, Akinul Islam and Arnob, Arjun Kumar Bose},
	month = mar,
	year = {2024},
	pages = {109--117},
}

@misc{william_wolberg_breast_1995,
	title = {Breast {Cancer} {Wisconsin} ({Prognostic})},
	author = {{William Wolberg} and {W. Street} and {Olvi Mangasarian}},
	year = {1995},
}

@article{mallat_theory_1989,
	title = {A theory for multiresolution signal decomposition: the wavelet representation},
	volume = {11},
	issn = {1939-3539},
	shorttitle = {A theory for multiresolution signal decomposition},
	number = {7},
	urldate = {2025-05-08},
	journal = {IEEE Transactions on Pattern Analysis and Machine Intelligence},
	author = {Mallat, S.G.},
	month = jul,
	year = {1989},
	pages = {674--693},
}

@article{davis_comparison_1980,
	title = {Comparison of parametric representations for monosyllabic word recognition in continuously spoken sentences},
	volume = {28},
	issn = {0096-3518},
	number = {4},
	urldate = {2025-05-08},
	journal = {IEEE Transactions on Acoustics, Speech, and Signal Processing},
	author = {Davis, S. and Mermelstein, P.},
	month = aug,
	year = {1980},
	pages = {357--366},
}

@inproceedings{h_krones_dual_2022,
	title = {Dual {Bayesian} {ResNet}: {A} {Deep} {Learning} {Approach} to {Heart} {Murmur} {Detection}},
	shorttitle = {Dual {Bayesian} {ResNet}},
	language = {en},
	urldate = {2025-05-08},
	author = {H. Krones, "Felix and Walker, Benjamin and Mahdi, Adam and Kiskin, Ivan and Lyons, Terry and Parsons", Guy},
	month = dec,
	year = {2022},
}

@article{shuvo_nrc-net_2023,
	title = {{NRC}-{Net}: {Automated} noise robust cardio net for detecting valvular cardiac diseases using optimum transformation method with heart sound signals},
	volume = {86},
	issn = {1746-8094},
	shorttitle = {{NRC}-{Net}},
	urldate = {2025-05-08},
	journal = {Biomedical Signal Processing and Control},
	author = {Shuvo, Samiul Based and Alam, Syed Samiul and Ayman, Syeda Umme and Chakma, Arbil and Barua, Prabal Datta and Acharya, U Rajendra},
	month = sep,
	year = {2023},
	pages = {105272},
}

@misc{han_enact-heart_2025,
	title = {{ENACT}-{Heart} -- {ENsemble}-based {Assessment} {Using} {CNN} and {Transformer} on {Heart} {Sounds}},
	urldate = {2025-05-08},
	publisher = {arXiv},
	author = {Han, Jiho and Shaout, Adnan},
	month = feb,
	year = {2025},
	note = {arXiv:2502.16914 [cs]},
}

@article{li_lightweight_2021,
	title = {Lightweight {End}-to-{End} {Neural} {Network} {Model} for {Automatic} {Heart} {Sound} {Classification}},
	volume = {12},
	issn = {2078-2489},
	language = {en},
	number = {2},
	urldate = {2025-05-08},
	journal = {Information},
	publisher = {Multidisciplinary Digital Publishing Institute},
	author = {Li, Tao and Yin, Yibo and Ma, Kainan and Zhang, Sitao and Liu, Ming},
	month = feb,
	year = {2021},
	note = {Number: 2},
	pages = {54},
}

@inproceedings{song_cutting_2023,
	title = {Cutting {Weights} of {Deep} {Learning} {Models} for {Heart} {Sound} {Classification}: {Introducing} a {Knowledge} {Distillation} {Approach}},
	issn = {2694-0604},
	shorttitle = {Cutting {Weights} of {Deep} {Learning} {Models} for {Heart} {Sound} {Classification}},
	urldate = {2025-05-08},
	booktitle = {2023 45th {Annual} {International} {Conference} of the {IEEE} {Engineering} in {Medicine} \& {Biology} {Society} ({EMBC})},
	author = {Song, Zikai and Zhu, Lixian and Wang, Yiyan and Sun, Mengkai and Qian, Kun and Hu, Bin and Yamamoto, Yoshiharu and Schuller, Björn W.},
	month = jul,
	year = {2023},
	pages = {1--4},
}

@inproceedings{clifford_classification_2016,
	title = {Classification of {Normal}/{Abnormal} {Heart} {Sound} {Recordings}: the {PhysioNet}/{Computing} in {Cardiology} {Challenge} 2016},
	shorttitle = {Classification of {Normal}/{Abnormal} {Heart} {Sound} {Recordings}},
	language = {en},
	urldate = {2024-07-14},
	booktitle = {2016 {Computing} in {Cardiology} {Conference}},
	author = {Clifford, Gari and Liu, Chengyu and Springer, David and Moody, Benjamin and Li, Qiao and Abad, Ricardo and Millet, Jose and Silva, Ikaro and Johnson, Alistair and Mark, Roger},
	month = sep,
	year = {2016},
}

@inproceedings{imani_hdcluster_2019,
	title = {{HDCluster}: {An} {Accurate} {Clustering} {Using} {Brain}-{Inspired} {High}-{Dimensional} {Computing}},
	issn = {1558-1101},
	shorttitle = {{HDCluster}},
	urldate = {2025-01-12},
	booktitle = {2019 {Design}, {Automation} \& {Test} in {Europe} {Conference} \& {Exhibition} ({DATE})},
	author = {Imani, Mohsen and Kim, Yeseong and Worley, Thomas and Gupta, Saransh and Rosing, Tajana},
	month = mar,
	year = {2019},
	pages = {1591--1594},
}

@article{masse_olfactory_2009,
	title = {Olfactory {Information} {Processing} in {Drosophila}},
	volume = {19},
	issn = {09609822},
	language = {en},
	number = {16},
	urldate = {2025-01-12},
	journal = {Current Biology},
	author = {Masse, Nicolas Y. and Turner, Glenn C. and Jefferis, Gregory S.X.E.},
	month = aug,
	year = {2009},
	pages = {R700--R713},
}

@article{thomas_theoretical_2021,
	title = {A {Theoretical} {Perspective} on {Hyperdimensional} {Computing}},
	volume = {72},
	copyright = {Copyright (c)},
	issn = {1076-9757},
	language = {en},
	urldate = {2025-01-08},
	journal = {Journal of Artificial Intelligence Research},
	author = {Thomas, Anthony and Dasgupta, Sanjoy and Rosing, Tajana},
	month = oct,
	year = {2021},
	pages = {215--249},
}

@book{zipes_braunwalds_2019,
	address = {Philadelphia, PA},
	edition = {Eleventh edition, international edition},
	title = {Braunwald's heart disease: a textbook of cardiovascular medicine},
	isbn = {978-0-323-46342-3 978-0-323-46299-0 978-0-323-55592-0},
	shorttitle = {Braunwald's heart disease},
	language = {en},
	publisher = {Elsevier},
	editor = {Zipes, Douglas P. and Libby, Peter and Bonow, Robert O. and Mann, Douglas L. and Tomaselli, Gordon F. and Braunwald, Eugene},
	year = {2019},
}

@article{kanerva_random_2000,
	series = {Erlbaum},
	title = {Random {Indexing} of {Text} {Samples} for {Latent} {Semantic} {Analysis}},
	language = {en},
	journal = {Proceedings of the 22nd Annual Conference of the Cognitive Science Society},
	author = {Kanerva, Pentti and Kristoferson, Jan and Holst, Anders},
	year = {2000},
	pages = {1036},
}

@misc{pentti_kanerva_hyperdimensional_2009,
	title = {Hyperdimensional {Computing}: {An} {Introduction} to {Computing} in {Distributed} {Representation} with {High}-{Dimensional} {Random} {Vectors}},
	publisher = {Cognitive Computation},
	author = {{Pentti Kanerva}},
	month = jan,
	year = {2009},
}
\bibliographystyle{splncs04}

\newpage
\appendix
\setcounter{theorem}{0}

\section{Theory}
\label{Appendix: Theory}

\subsection{Notations and Definitions}
\label{Appendix: Notations and Definitions}
\begin{definition}[Hyperspace]
\label{definition: Hyperspace}
The hyperspace $\mathcal{H}^D$ is the set of all bipolar hypervectors of dimension $D$, where each element takes a value from $\{+1, -1\}$. Formally,
\begin{equation}
\mathcal{H}^D := \{ -1, +1\}^D = \{ A \in \mathbb{R}^D \mid a_i \in \{+1, -1\},\ \forall i \in \{1, 2, \dots, D\} \}
\end{equation}
where $a_i$ denotes the $i$-th component of the hypervector $A$.

Unless otherwise specified, for a \textit{random hypervector} $A \in \mathcal{H}^D$, each component $a_i$ is drawn independently and identically distributed (i.i.d.) from a binary distribution over $\{-1, +1\}$ with equal probability:
\begin{equation}
\mathbb{P}(a_i = +1) = \mathbb{P}(a_i = -1) = 0.5, \quad \forall i \in \{1, \dots, d\}.
\end{equation}

\end{definition}

\begin{definition}[Elementary Functions]
\label{definition: Elementary Functions}
Given any hypervectors $A, B \in \mathcal{H}^D$, we define the following elementary functions:

\begin{itemize}
    \item The number of $+1$ elements in $A$:
    \begin{equation}
    \text{pos}(A) := \sum_{i=1}^{D} \mathbb{I}(a_i = +1)
    \end{equation}

    \item The number of $-1$ elements in $A$:
    \begin{equation}
    \text{neg}(A) := \sum_{i=1}^{D} \mathbb{I}(a_i = -1)
    \end{equation}

    \label{opr:binding}
    \item Binding of A and B:
    \begin{equation}
    A \otimes B := A \times B
    \end{equation}

    \label{opr:bundling}
    \item Bundling of A and B:
    \begin{equation}
    A \oplus B := [A + B]
    \end{equation}

\end{itemize}
where $\mathbb{I}(\cdot)$ denotes the indicator function that returns $1$ if the condition is true and $0$ otherwise, $\times$ is an element-wise multiplication, $+$ is an element-wise addition, and $[\cdot]$ is the element-wise majority function that outputs $-1$ for negative sum, $+1$ for positive sum, and equally random drawn from $\{-1, +1\}$ for zero sum.
\end{definition}

\begin{definition}[Hamming Distance]
\label{definition: Hamming Distance}
    Given two hypervectors $A, B \in \mathcal{H}^D$, the Hamming distance between them is defined as the ratio of positions where the corresponding components differ. Formally:
    
    \begin{equation}
    d_{H}: \mathcal{H}^D \times \mathcal{H}^D \rightarrow [0, 1], \ \ d_H(A, B) = \frac{1}{D}\sum_{i=1}^D \mathbb{I}(a_i \neq b_i) = \frac{1}{D}\text{neg}(A\otimes B).
    \end{equation}
\end{definition}

\begin{definition}[Level Set]
\label{definition: Level Set} We define a set of M hypervectors $\mathcal{L} = \{L^{(j)}\}_{j \in \{1, 2, \cdots, M\}} \subset \mathcal{H}^D$ as a level set if it can be generated by:
\begin{itemize}
    \item Randomly sample a base hypervector $L^{(1)} \in \mathcal{H}^D$
    \item Initialize an empty set $\mathcal{B} = \emptyset$ to record flipped bit positions.
    \item For each $i \in \{2, 3, \dots, M\}$:
    \begin{itemize}
        \item Randomly select $\frac{D}{M-1}$ positions from $\{1, \dots, D\} \setminus \mathcal{B}$ (i.e., bits not yet flipped). To avoid unnecessary complexity, we assume that $M-1$ divides $D$ exactly here, i.e., $\frac{D}{M-1} \in \mathbb{N}$.
        \item Flip the selected bits in $L^{(i-1)}$ to obtain $L^{(i)}$;
        \item Update $\mathcal{B}$ to include the newly flipped bit positions.
    \end{itemize}
\end{itemize}
In the constructed level set $\mathcal{L}$, the Hamming distance between any two levels $L^{(i)}$ and $L^{(j)}$ satisfies
\[
d_H(L^{(i)}, L^{(j)}) = \frac{|i-j|}{M-1}.
\]

\end{definition}

\begin{definition}[Mapping with Levels]
\label{definition: Mapping with Levels} 
Let $\mathcal{X} \subseteq \mathbb{R}^d$ be a continuous feature vector space.  
Define independently sampled $d$ random hypervectors $\mathcal{ID} = \{ID^{(1)}, ID^{(2)}, \dots, ID^{(d)}\} \subset \mathcal{H}^D$ and a random level set $\mathcal{L} = \{L^{(1)}, L^{(2)}, \dots, L^{(M)}\} \subset \mathcal{H}^D$. Next, for each feature dimension $n \in \{1, \dots, d\}$, define $M-1$ real-valued thresholds $\theta_{n, 1} < \theta_{n, 2} < \dots < \theta_{n, M-1}$, partitioning $\mathbb{R}$ into $M$ intervals:

    \[
    I_0 = (-\infty, \theta_{n, 1}),\quad I_1 = [\theta_{i, 1}, \theta_{n, 2}),\quad \dots,\quad I_{M-1} = [\theta_{n,M-1}, +\infty).
    \]
Based on the intervals defined by $\Theta = \{\theta_{n, m}\}_{n \in \{1, \dots, d\}, m \in \{1, \dots, M\}} $, we select the unique $m = l_n(x_i[n])$ for $x_i[n] \in \mathbb{R}$ that $x_i[n] \in I_m$.

In this task, the mapping function for feature $n$ is defined by

\[
l_n(x) = 
\left\{
\begin{array}{ll}
1 & x \in (-\infty, \theta_{n, \alpha}) \\
\lfloor \frac{x-\theta_{n, \alpha}}{\theta_{n, \beta} - \theta_{n, \alpha}} \times M + 1 \rfloor   & x \in [\theta_{n, \alpha}, \theta_{n, \beta}) \\
M & x \in [\theta_{n, \beta}, \infty)
\end{array}
\right.,
\]
where $\theta_{n, \alpha}$ and $\theta_{n, \beta}$ denote the 2\% and 98\% quantiles of values in feature $n$, respectively.


\end{definition}

\subsection{Lemmas and Propositions }
\label{Appendix: Useful Lemma}
\begin{lemma}[Hamming Distance Between Product and Multiplier]
\label{lemma: Hamming Distance between product and multiplier}
Given hypervectors $A, B \in \mathcal{H}^D$, the Hamming distance $d_H(A, A * B)$ depends only on $B$. Specifically, 
    \begin{equation}
    d_H(A, A \times B) = \frac{1}{D}\text{neg}(B).
    \end{equation}

\end{lemma}
\begin{proof}
By the definition of Binding, the $i$-th element of $A * B$ is $a_i * b_i$, which behaves as:
\[
a_i * b_i =
\begin{cases}
a_i, & \text{if } b_i = +1, \\
-a_i, & \text{if } b_i = -1.
\end{cases}
\ \ \ 
\text{or} \ \ 
a_i \neq (a_i * b_i) \iff b_i = -1
\]

Therefore,
\[
d_H(A, A * B) = \frac{1}{D}\sum_{i=1}^D \mathbb{I}(a_i \neq (a_i * b_i)) = \frac{1}{D}\sum_{i=1}^D \mathbb{I}(b_i = -1) = \frac{1}{D}\text{neg}(B).
\]

\end{proof}

\begin{lemma}[Hamming Distance Preservation Under Multiplication]
\label{lemma: Hamming Distance Preservation under Multiplication}
Given hypervectors $A, B, C \in \mathcal{H}^D$, we have
    \begin{equation}
    d_H(A * B, B * C) = d_H(A, C).
    \end{equation}

\end{lemma}
\begin{proof}
By the definition of Binding, the $i$-th elements of $A * B$ and $B * C$ are $a_i * b_i$ and $b_i * c_i$.

Since $b_i \neq 0$, we have $a_i * b_i \neq b_i * c_i \iff a_i \neq c_i$. 
Therefore, 
\[
d_H(A * B, B * C) = \frac{1}{D}\sum_{i=1}^D \mathbb{I}(a_i * b_i \neq b_i * c_i) = \frac{1}{D}\sum_{i=1}^D \mathbb{I}(a_i \neq c_i) = d_H(A, C).
\]

\end{proof}

\begin{lemma}[Hamming Distance Between Two Random Hypervectors]
\label{lemma: Hamming distance between two Random Hypervectors}
Let $A, B \in \mathcal{H}^D$ be two random hypervectors,
Then, for any $\epsilon > 0$,
    \begin{equation}
    \lim_{D \to \infty} \mathbb{P}\left(\left|d_H(A, B) - 0.5\right| > \epsilon\right) = 0.
    \end{equation}

\end{lemma}
\begin{proof}
Define random variables $X_i := \mathbb{I}(a_i \neq b_i)$, where $X_i = 1$ if $a_i \neq b_i$ and $X_i = 0$ otherwise. By definition,
\[
d_H(A, B) = \frac{1}{D} \sum_{i=1}^D X_i.
\]

Since $A$ and $B$ are random hypervectors, their components $a_i$ and $b_i$ are drawn i.i.d. from a binary distribution over $\{-1, +1\}$ with equal probability. Thererfore, 
\[
\mathbb{P}(a_i \neq b_i) = 0.5, \quad \mathbb{P}(a_i = b_i) = 0.5,
\]
indicating that each $X_i$ is an independent Bernoulli random variable with  $\mathbb{E}[X_i] = 0.5$. Furthermore, we know that the average of i.i.d. sequence $\{X_i\}$, $\frac{1}{D} \sum_{i=1}^D X_i$ follows a scaled Binomial distribution:
\[
d_H(A, B) = \frac{1}{D} \sum_{i=1}^D X_i \sim \frac{1}{D} \text{Binomial}(D, \frac{1}{2}).
\]

Applying the Weak Law of Large Numbers to the i.i.d. sequence $\{X_i\}$, for any $\epsilon > 0$,
\[
\lim_{D \to \infty} \mathbb{P}\left( \left| \frac{1}{D} \sum_{i=1}^D X_i - 0.5 \right| > \epsilon \right) = 0.
\]

Thus,
\[
\lim_{D \to \infty} \mathbb{P}\left( \left| d_H(A, B) - 0.5 \right| > \epsilon \right) = 0.
\]

\end{proof}

\begin{lemma}[Orthogonality of Binding]
\label{lemma: }
Let $A, B \in \mathcal{H}^D$ be two random hypervectors,
Then, for any $\epsilon > 0$,
\[
\lim_{D \to \infty} \mathbb{P}\left(\left|d_H(A, A * B) - 0.5\right| > \epsilon\right) = 0.
\]

\end{lemma}
\begin{proof}
From the Lemma~\ref{lemma: Hamming Distance between product and multiplier}, we know that
\[
d_H(A, A * B) = \frac{1}{D} \sum_{i=1}^D \mathbb{I}(b_i = -1) = \frac{\text{neg}(B)}{D}.
\]

Since $B$ is a random hypervector, each component $b_i$ is drawn i.i.d. from a binary distribution over $\{-1, +1\}$ with equal probability. Again, by the Weak Law of Large Numbers,
\[
\lim_{D \to \infty} \mathbb{P}\left(\left|\frac{\text{neg}(B)}{D} - 0.5\right| > \epsilon\right) = 0,
\]
thus leading to
\[
\lim_{D \to \infty} \mathbb{P}\left(\left|d_H(A, A * B) - 0.5\right| > \epsilon\right) = 0.
\]

\end{proof}

\begin{proposition}[Hamming Distance with sufficiently large $D$]
\label{proposition: statistic hamming distance}
Let two hypervectors $A, B \in \mathcal{H}^D$ that satisfy
\[
    d_H(A, B) = \delta,
\]
and the random variable $Z_i$ indicates the situation of the $i^{th}$ bit
\[
    Z_i = \mathbb{I}\left( a_i \neq b_i \right).
\]

With sufficiently large $D$, the $Z_i$ can be approximately viewed as an i.i.d. Bernoulli distribution

\[
Z_i \sim_{i.i.d.} \text{Bernoulli}(\delta), \forall i \in \{1,2, \ldots,D\}.
\]

\end{proposition}

\begin{proof}
With the definition of Hamming distance, we have
\[
\frac{1}{D} \sum_{i=1}^D Z_i = d_H(A, B) = \delta.
\]

Observe that this is an empirical mean of $\{Z_i\}$ over $D$ bits. Notice that the $Z_i$ are NOT independent — their sum is fixed to be exactly $\delta D$. However, we make the following observations and assumptions to justify such a approximation:
\begin{itemize}
    \item With sufficiently large D, the possibility of the Hamming distance of A and B completely falls into an arbitrarily small interval around the $\delta$ is almost 1:
    \[
    \mathbb{P}\left( \left| \frac{1}{D} \sum_{i=1}^D Z_i - d \right| > \epsilon \right) \leq 2 \exp(-2 D \epsilon^2).
    \]

    \item If $A$ and $B$ are generated randomly and conditioned on their Hamming distance being $dD$, then the mismatch positions mentioned above are uniformly random for every index;
\end{itemize}

\end{proof}

To summarize, this proposition provides a new perspective to define the Hamming distance in the large hyperspace. With sufficiently large dimensionality $D$, Hamming distance between hypervectors can be interpreted statistically as the empirical mean of i.i.d. Bernoulli random variables, yielding a probabilistic characterization of similarity.

\subsection{Theorems}
\label{Appendix: Theorems}

\begin{theorem}[Robustness to Input Noise]
\label{prf:robust_to_input_noise}
Let $s^{(1)}$ be a feature vector and $s^{(2)}$ its noisy variant, with their Sample-HVs denoted as $S^{(1)}, S^{(2)} \in \{-1, 1\}^D$ according to Equation~\ref{eqn: mapping}.
Suppose that for feature dimensions $n \in \{1, \dots, d\}$,
\begin{equation}
\frac{\left| s_n^{(1)} - s_n^{(2)} \right|}{\Delta_n} \leq \delta,
\end{equation}
where $\Delta_n$ is the range of the $n^{\mathrm{th}}$ feature value, and $\delta \in [0, 1]$ denotes the maximum normalized perturbation.
Then, with a sufficiently large $D$, the expected upper-bound of the Hamming distance between $S^{(1)}$ and $S^{(2)}$ converges to a monotonically increasing function $g(\delta)$:
\begin{equation}
    \lim_{D \to \infty} \mathbb{E}\left[\sup d_H\left( S^{(1)}, S^{(2)} \right) \right] = g(\delta).
\end{equation}
\end{theorem}

\begin{proof}

We consider the $L^{(l_n(s_n^{(1)})}, L^{(l_n(s_n^{(2)})}$ first.
Since $\frac{\left| s^{(1)}_n - s^{(2)}_n \right|}{\Delta_n} \leq \delta$, we can bound the difference of their corresponding Level Hypervectors. By definition~\ref{definition: Level Set}, we have

\begin{equation}
    \begin{aligned}
\quad &d_H\left( L^{(l_n(s^{(1)}_n))}, L^{(l_n(s^{(2)}_n))} \right) \\
&=   \frac{\lvert l_n(s^{(1)}_n) - l_n(s^{(2)}_n)\rvert}{M-1} \\
&\leq \frac{1}{M-1} \left( \frac{\delta \Delta_n}{\theta_{i, M-1} - \theta_{i, 1}}(M-2) + 2\right) \\
&\leq \frac{M}{M-1}\delta + \frac{2}{M-1} = \delta_d,
    \end{aligned}
\end{equation}

With the random flipping performed by $\otimes ID^{(i)}$, $\{L^{(l_n(s^{(1/2)}_n))} \otimes ID^{(i)}\}$ can be viewed as a new set of randomly generated hypervectors, denoted as $\{ T^{i(1/2)} \}$, which satisfy the following relationship with Lemma~\ref{lemma: Hamming Distance Preservation under Multiplication}
\begin{equation}
\begin{aligned}
&\quad d_H\left( T^{i(1)}, T^{i(2)} \right) \\
&= d_H\left( L^{(l_n(s^{(1)}_n))} \otimes ID^{(i)}, L^{(l_n(s^{(2)}_n))}  \otimes ID^{(i)} \right)  \\
& = d_H\left( L^{(l_n(s^{(1)}_n))}, L^{(l_n(s^{(2)}_n))} \right) \leq \delta_d.
\end{aligned}
\end{equation}

To estimate the upper bound, we treat the inequality as an equality. Furthermore, with Proposition~\ref{proposition: statistic hamming distance}, we consider the Hamming distance constraint as a statistically condition, so that for every $T^{i(1)}$ and $T^{i(2)}$ pair, the equality situation at the $j^{th}$ index follows a Bernoulli distribution
\[
\mathbb{I} \left( T_j^{i(1)} \neq T_j^{i(2)} \right) \sim_{i.i.d.} \text{Bernoulli}(\delta_d), \forall j \in \{1,2, \ldots,D\}.
\]

Next, we consider the situation on the $j^{th}$ index of $S^{(1)} = \bigoplus_{i=1}^{d} T^{i(1)}$ and $S^{(2)} = \bigoplus_{i=1}^{d} T^{i(2)}$ as a random variable $Z_j$ that satisfies

\[
Z_j = \mathbb{I}\left(S^{(1)}_{(j)} \neq S^{(2)}_{(j)}\right)
\]

Let $p_j^{(1/2)}$ be the number of $+1$ in the $\{T_j^{i(1/2)}\}_{i=\{1, 2, \ldots, d\}}$, so that the probability of $Z_j = 1$ is

\[
\quad P(Z_j=1) = \sum_{n=0}^{d} P(Z_j = 1 | p_j^{(1)} = n) \times P(p_j^{(1)}=n).
\]

Specifically, by considering the $\{T_j^{i(1)}\}_{i=\{1, 2, \ldots, d\}}$ follow the i.i.d. \text{Bernoulli}(0.5), we have

\[
P(p_j^{(1)}=n) = \frac{C_d^n}{2^d}.
\]

Consider the situation where $n < \frac{d}{2}$ and $d$ is an odd number to avoid unnecessary complexity, we have the number of $+1$ in the second sample follow a combined Binomial distribution

\[
p_j^{(2)}|_{p_j^{(1)}=n} \sim \text{Binomial}\left(n, 1-\delta_d \right) + \text{Binomial}\left(d-n, \delta_d \right).
\]
Therefore, 

\[
P(Z_j = 1 | p_j^{(1)} = n) = 
\sum_{k =  \frac{d+1}{2}}^{d} \sum_{i = \max(0,\, k - d + n)}^{\min(n,\, k)} 
C_n^i  \cdot 
C_{d - n}^{k - i} \cdot  \delta^{n + k - 2i}(1 - \delta)^{d - n - k + 2i}.
\]

For situation that $n > \frac{d}{2}$, we obtain an entirely analogous result
\[
P(Z_j = 1 | p_j^{(1)} = n) = 
\sum_{k =  0}^{\frac{d-1}{2}} \sum_{i = \max(0,\, k - d + n)}^{\min(n,\, k)} 
C_n^i  \cdot 
C_{d - n}^{k - i} \cdot  \delta^{n + k - 2i}(1 - \delta)^{d - n - k + 2i}.
\]

With the derived results, we can calculate the probability of $P(Z_j = 1)$ as a function of $\delta_d$. 
\begin{equation}
\label{eqn: zj=1 prob}
g(\delta) = p(\delta_d) = p(\frac{M}{M-1}\delta + \frac{2}{M-1})
\end{equation}


Finally, we consider the average with all the indices make the distribution arbitrarily close to the expectation, which is the probability we just calculated in Eqn.~\ref{eqn: zj=1 prob}, thus complete the proof.

\end{proof}

This theorem establishes an approximate form of "continuity", showing that the designed mapping ensures that when the input noise is small, the distance between the corresponding mapped outputs is also bounded. In particular, this design exhibits better rejection properties against small perturbations with sufficiently large dimensionality $D$ and relatively large features $d$.


\begin{theorem}[Distance Between Cluster Prototype and Constituents]
\label{prf:proximity_guarantee}
Let $S^{(1)}, S^{(2)}, \dots, S^{(N)} \in \mathcal{H}^D$ be independently sampled random hypervectors. Define their sum as $C = [S^{(1)} + S^{(2)} + \dots + S^{(N)}]$. 
As $D \to \infty$, for any random hypervector $S^* \in \mathcal{H}^D$, index $n \in \{1, \dots, N\}$, the Hamming distance between $C$ and any component $S^{(n)}$ satisfies 
\[
P\left(d_H(C, S^{(n)}) < d_H(C, S^*)\right) \rightarrow 1.
\]

\end{theorem}

\begin{proof}
We begin with analyzing the probability that the $i^{th}$ index of $C$ and $S^{(n)}$ is different.

Let $Z_i^{(n)} = \sum_{k \neq n} S_i^{(k)}$, where $S_i^{(k)}$ is the $i^{th}$ value of $S^{(k)}$, so that:
\[
C_i = \left[ \sum_{k=1}^N S_i^{(k)} \right] = \left[  S_i^{(n)} + \sum_{k \neq n} S_i^{(k)} \right] = \left[ S_i^{(n)} + Z_i^{(n)} \right].
\]

Since the random hypervectors are independently sampled, the $N-1$ terms $S_i^{(k)}$ ($k \neq n$) are independent random variables satisfying:
\[
\mathbb{P}(S_i^{(k)} = +1) = \mathbb{P}(S_i^{(k)} = -1) = 0.5.
\]

Since $\frac{S_i^{(k)} + 1}{2} \sim \text{Bernoulli}(0.5)$, their i.i.d. sum $Z_i^{(n)}$ follows a shifted binomial distribution:
\[
Z_i^{(n)} \sim 2 * \text{Binomial}(N-1, 0.5) - (N-1)
\]

Given $S_i^{(n)} = +1$, we have $C_i = [1 + Z_i^{(n)}]$. Without losing generality, we only consider the situation when $N-1$ is an even number, which lead to
\[
P(C_i \neq S_i^{(n)}) = P( [1 + Z_i^{(n)}] \neq 1) = P(Z_i^{(n)} < -1) + \frac{1}{2}P(Z_i^{(n)} = -1) = P(Z_i^{(n)} < -1),
\]
where $P(Z_i^{(n)} = -1) = 0$ since the sum can only take even values.

With the distribution of $Z_i^{(n)}$, we can further calculate the specific probability:
\[
P(C_i \neq S_i^{(n)}) = \frac{1}{2} -  2^{-N} C_{N-1}^{\frac{N-1}{2}}.
\]

By symmetry, the same calculation applies if $S_i^{(n)} = -1$.
Notice that $\mathbb{I}(C_i \neq S_i^{(n)})$ presents a Bernoulli distribution and this process is independently carried out for every index, we can derive the distribution of the Hamming distance $d_H(C, S^{(n)})$:
\[
d_H(C, S^{(n)}) = \frac{1}{D}\sum_{i=1}^D \mathbb{I}(C_i \neq S_i^{(n)}) \sim \frac{1}{D} \text{Binomial}(D, \frac{1}{2} - p(N))
\]
where $p(N) = 2^{-N} C_{N-1}^{\frac{N-1}{2}}$.

Due to independently sampling, we can view $S^*$ and $B$ just as two random hypervectors in the hyperspace, so we can apply the result from Lemma~\ref{lemma: Hamming distance between two Random Hypervectors}, we know that

\[
d_H(C, S^*)  \sim \frac{1}{D} \text{Binomial}(D, \frac{1}{2}).
\]

With central limit theorem, these two distributions converge to normal distributions as $D \rightarrow \infty$

\[
d_H(C, S^*)  \sim \mathcal{N}(\frac{1}{2}, \frac{1}{4}D^{-1}), \ \ \ d_H(C, S^{(n)}\sim \mathcal{N}(\frac{1}{2}-p(N), (\frac{1}{4}-p^2(N))D^{-1}).
\]

Based on that, we have,
\begin{equation}
\begin{aligned}
       & \quad P\left(d_H(C, S^{(n)}) - d_H(C, S^*) \geq 0\right) \\
       & = P\left((d_H(C, S^{(n)}) -\frac{1}{2} - \frac{p(N)}{2}) - (d_H(C, S) -\frac{1}{2} - \frac{p(N)}{2}) \geq 0\right) \\
       & \leq P\left((d_H(C, S^{(n)}) -\frac{1}{2} - \frac{p(N)}{2}) \geq 0 \right) + P\left((d_H(C, S) -\frac{1}{2} - \frac{p(N)}{2}) \leq 0\right) \\
       & \leq P\left( \left| d_H(C, S^{(n)}) -\frac{1}{2} - p(N) \right| \geq \frac{p(N)}{2})  \right) + P\left(\left| d_H(C, S) -\frac{1}{2} \right| \geq \frac{p(N)}{2} \right) \\
       & = 2 * \Phi(-p(N)D^{\frac{1}{2}}) + 2 * \Phi(-\frac{p(N)}{\sqrt{1-4p^2(N)}}D^{\frac{1}{2}}).
\end{aligned}
\end{equation}

We can select appropriate $D = D(p(N))$ to allow the right-hand side of the equation to approach zero arbitrarily by letting $p(N)D^{\frac{1}{2}} \rightarrow \infty$. Thus we have the result
\[
P\left(d_H(C, S^{(n)}) < d_H(C, S^*)\right) \rightarrow 1,
\]
where $\Phi$ is the Cumulative Distribution Function of the standard normal distribution.

Moreover, we further know that the convergence is roughly characterized $\Phi(-p(N)D^{\frac{1}{2}})$.

Notice that $p(N) \rightarrow 0$ as $N \rightarrow \infty$, and the dimension $D$ is required to be dependent on $p(N)$ to make the scale become effective, we apply $D \gg N$ for the implementation. Intuitively speaking, we require a sufficiently large $N$ to divide this two normal peak apart given a small $p(N)$.

\end{proof}

\begin{theorem}[Robustness to Hardware Error]
\label{prf:robust_to_hardware_error}
Assume we have a sample hypervector $S$ and two cluster hypervectors $C_1$ and $C_2$, whose initial Hamming distances satisfy:
\[
d_H(S, C_2) - d_H(S, C_1) = \epsilon > 0.
\]

We randomly flip a proportion $p$ ($p < 0.5$) of the bits in both $C_1$ and $C_2$, yielding two new hypervectors $C'_1$ and $C'_2$. As $D \rightarrow \infty$, we have

\[
P\big( d_H(S, C'_1) < d_H(S, C'_2) \big) \to 1.
\]
\end{theorem}

\begin{proof}
We begin with the situation at the $i^{th}$ index. Let $X_i$ and $Y_i$ be the indicator variables denoting whether the $i$-th bit differs from $S$ after corruption:
\[
X_i = \mathbf{1}_{\{(C_1')_i \neq S_i\}}, \qquad Y_i = \mathbf{1}_{\{(C_2')_i \neq S_i\}},
\]
so that the post-corruption Hamming distances are:
\[
d_H(S, C_1') = \frac{1}{D} \sum_{i=1}^D X_i, \qquad
d_H(S, C_2') = \frac{1}{D} \sum_{i=1}^D Y_i.
\]

Next, we analyze the expectation of each $X_i$ and $Y_i$. With the view of Proposition~\ref{proposition: statistic hamming distance}, with a sufficiently large $D$, we can transfer the Hamming distance $d$ as a random event with probability $d$:
\begin{itemize}
    \item For each $i$, if $C_1$ and $S$ originally differ at bit $i$, which happens with probability $d_1$, then flipping that bit with probability $p$ yields:
    \[
    \mathbb{P}(X_i = 1 \mid \text{originally different}) \rightarrow 1 - p,
    \qquad \mathbb{P}(X_i = 0 \mid \text{originally different}) \rightarrow p.
    \]
    \item If they originally agree (probability $1 - d_1$), then:
    \[
    \mathbb{P}(X_i = 1 \mid \text{originally same}) \rightarrow p,
    \qquad \mathbb{P}(X_i = 0 \mid \text{originally same}) \rightarrow 1 - p.
    \]
\end{itemize}

Hence, let $d_1 = d_H(S, C_1)$ and $d_2 = d_H(S, C_2)$, the expectation becomes:
\[
\mathbb{E}[X_i] \rightarrow d_1 (1 - p) + (1 - d_1) p = d_1 (1 - 2p) + p.
\]
Similarly,
\[
\mathbb{E}[Y_i] = d_2 (1 - 2p) + p.
\]

Since $d_2 = d_1 + \epsilon$ and $1 - 2p > 0$, we have:
\[
\mathbb{E}[Y_i] - \mathbb{E}[X_i] = (d_2 - d_1)(1 - 2p) = \epsilon (1 - 2p) > 0.
\]

Now define the total difference in post-corruption Hamming distances:
\[
Z = \sum_{i=1}^D (Y_i - X_i).
\]
Then:
\[
\mathbb{E}[Z] = D \cdot (\mathbb{E}[Y_i] - \mathbb{E}[X_i]) = D \cdot \epsilon (1 - 2p).
\]

Because $X_i, Y_i$ are bounded, independent random variables, and the variance of each term is bounded, we can apply Hoeffding's inequality to show that:
\[
\mathbb{P}\left( Z < 0 \right) \leq \exp(-cD)
\]
for some constant $c > 0$. This implies:
\[
\mathbb{P}(d_H(S, C_1') < d_H(S, C_2')) = \mathbb{P}(Z > 0) \to 1 \quad \text{as } D \to \infty.
\]

\end{proof}

\end{document}